\documentclass[preprint,12pt]{elsarticle}

\usepackage[utf8]{inputenc}
\usepackage{amsmath,amssymb,amsthm}
\usepackage{graphicx}
\usepackage{float}
\usepackage{xcolor}
\usepackage{comment}
\usepackage{hyperref}

\newif\ifmarked
\markedfalse
\theoremstyle{plain}
\newtheorem{theorem}{Theorem}[section]
\newtheorem{lemma}[theorem]{Lemma}
\newtheorem{proposition}[theorem]{Proposition}
\newtheorem{corollary}[theorem]{Corollary}
\theoremstyle{definition}
\newtheorem{definition}[theorem]{Definition}

\theoremstyle{remark}

\newtheorem{principle}[theorem]{Principle}

\journal{arXiv preprint}

\begin{document}

\begin{frontmatter}

\title{The Existential Theory of Research: Why Discovery Is Hard}

\author{Angshul Majumdar}
\address{Indraprastha Institute of Information Technology Delhi, India}

\begin{abstract}
Can scientific discovery be made arbitrarily easy—by choosing the right representation, collecting enough data, and deploying sufficiently powerful algorithms? This paper argues that the answer is fundamentally negative. We introduce the Existential Theory of Research (ETR), a formal framework that models discovery as the recovery of structured explanations under constraints of representation, observation, and computation. Within this framework, we show that these three components cannot be simultaneously optimized: no method can guarantee universally simple explanations, arbitrarily compressed observations, and efficient exact inference. This limitation is not model-specific, but arises from a synthesis of uncertainty principles in sparse representation, sample complexity bounds in high-dimensional recovery, and the computational hardness of exact inference. We further show that representation mismatch alone can inflate intrinsic simplicity into apparent complexity, rendering otherwise tractable problems observationally and computationally prohibitive. To quantify these effects, we introduce an uncertainty functional that captures the joint difficulty of discovery. The results suggest that scientific difficulty is not accidental, but a structural consequence of the geometry and complexity of inference.
\end{abstract}


\end{frontmatter}

\section{Introduction}
\label{sec:introduction}

Can discovery be made arbitrarily easy? Given a sufficiently expressive representation, enough data, and a powerful inference procedure, it is tempting to believe that recovering the correct explanation is merely a matter of engineering. This paper argues that such a view is fundamentally incomplete.

At a technical level, many modern inference problems share a common structure: an underlying phenomenon admits a structured (often low-complexity) representation, observations are partial or compressed, and recovery is subject to computational constraints. These aspects have been studied extensively in isolation—through sparse representation theory \cite{DonohoHuo2001,EladBruckstein2002}, sample complexity in high-dimensional recovery \cite{CandesRombergTao2006,RudelsonVershynin2008}, and the computational hardness of exact inference \cite{Natarajan1995,Tropp2004}. What remains less explicit is how these constraints interact when considered jointly.

While we use the term ``scientific discovery'' throughout, the framework developed here does not depend on any discipline-specific assumptions. Instead, we study a class of inference problems in which an underlying phenomenon admits a structured representation, observations are partial, and recovery is constrained by computation. This abstraction captures a broad range of settings, including equation discovery, variable selection, and high-dimensional model inference.

To formalize this perspective, we introduce the \emph{Existential Theory of Research} (ETR), which models discovery through a triplet $(\Psi,\Phi,\mathcal{A})$ consisting of a representation system $\Psi$, an observational operator $\Phi$, and an inference procedure $\mathcal{A}$. Within this framework, discovery is cast as a recovery problem: one seeks a low-complexity explanation $x$ from observations $y = \Phi x$. The representation determines what counts as a simple explanation, the observation operator determines what is distinguishable, and the inference procedure determines what is computationally accessible.

The central result of the paper establishes a fundamental limitation of this process. We show that no framework can simultaneously guarantee: (i) universally low-complexity representations across incompatible paradigms, (ii) reliable recovery from arbitrarily compressed observations, and (iii) exact and efficient inference for all admissible problems. This limitation arises from a synthesis of three well-known phenomena: uncertainty principles in sparse representation \cite{DonohoHuo2001,EladBruckstein2002}, sample complexity bounds for uniform recovery \cite{CandesRombergTao2006,RudelsonVershynin2008}, and the NP-hardness of exact sparse optimization \cite{Natarajan1995,Tropp2004}.

A second contribution is to show that these limitations may arise endogenously through representation mismatch. Even when a phenomenon admits a sparse description in an appropriate representation, an incorrect choice of representation can inflate its apparent complexity, degrade observational distinguishability, and render recovery computationally prohibitive. In this sense, difficulty need not reflect the intrinsic structure of the problem, but may instead arise from the interaction between representation and observation.

To quantify these effects, we introduce an uncertainty functional that aggregates representation complexity, observational geometry, and computational effort into a single measure of discovery difficulty. This functional provides a scalar characterization of the regimes identified later in the paper, ranging from stable recovery to non-identifiability.

\paragraph{Contributions.}
The contributions of this paper are as follows:
\begin{enumerate}
    \item We introduce the ETR framework, modeling discovery as a joint problem of representation, observation, and inference.
    \item We establish a fundamental limitation theorem showing that these components cannot be simultaneously optimized.
    \item We identify representation mismatch as an endogenous mechanism that inflates apparent complexity and degrades recoverability.
    \item We develop a quantitative uncertainty functional capturing the joint difficulty of discovery.
    \item We characterize regimes of inquiry in terms of identifiability, stability, and computational accessibility.
\end{enumerate}

The remainder of the paper is organized as follows. Section~\ref{sec:preliminaries} introduces notation and formal setup. Section~\ref{sec:theory_limits} develops the core theoretical results. Section~\ref{sec:implications} interprets these results in terms of regimes of inquiry. Section~\ref{sec:quantitative_etr} introduces a quantitative formulation of existential uncertainty. Section~\ref{sec:conclusion} concludes.
\section{Preliminaries and Problem Formulation}
\label{sec:preliminaries}

In this section, we introduce the notation and formal setup used throughout the paper. The objective is to define a minimal but precise framework within which the limits of scientific discovery can be analyzed.

\subsection{Representation and Scientific Models}
\label{subsec:representation}

Let $\mathcal{U} \subseteq \mathbb{R}^d$ denote the space of admissible phenomena (``scientific truths''). We assume that each $x \in \mathcal{U}$ admits a structured representation in a dictionary $\Psi \in \mathbb{R}^{d \times N}$, i.e.,
\begin{equation}
x = \Psi \alpha,
\end{equation}
for some coefficient vector $\alpha \in \mathbb{R}^N$.

\begin{definition}[Representation complexity]
\label{def:rep_complexity}
For $x \in \mathcal{U}$, define the representation complexity
\begin{equation}
K_\Psi(x)
:=
\min_{\alpha:\, x=\Psi \alpha} \|\alpha\|_0,
\end{equation}
where $\|\cdot\|_0$ denotes the support size.
\end{definition}

The quantity $K_\Psi(x)$ measures the intrinsic complexity of $x$ relative to the representation $\Psi$. Different choices of $\Psi$ correspond to different scientific descriptions (e.g., Fourier, wavelet, or feature-based representations).

\subsection{Observation Model}
\label{subsec:observation}

Observations are obtained through a linear sensing operator $\Phi \in \mathbb{R}^{m \times d}$:
\begin{equation}
y = \Phi x + e,
\label{eq:observation_model}
\end{equation}
where $y \in \mathbb{R}^m$ denotes the observed data and $e \in \mathbb{R}^m$ is an unknown noise vector satisfying
\begin{equation}
\|e\|_2 \le \epsilon.
\end{equation}

The regime of interest is underdetermined, i.e.,
\begin{equation}
m \ll d,
\end{equation}
so that recovery of $x$ from $y$ requires structural assumptions.

Combining representation and observation yields the effective sensing model
\begin{equation}
y = A \alpha + e,
\quad \text{where} \quad
A := \Phi \Psi \in \mathbb{R}^{m \times N}.
\label{eq:effective_model}
\end{equation}

\subsection{Inference Procedures}
\label{subsec:inference}

A reconstruction procedure $\mathcal{A}$ takes observations $y$ and returns an estimate $\hat{x} = \mathcal{A}(y)$, or equivalently an estimate $\hat{\alpha}$ satisfying $\hat{x} = \Psi \hat{\alpha}$.

Typical reconstruction procedures include:
\begin{itemize}
    \item $\ell_0$ minimization,
    \item convex relaxations (e.g., $\ell_1$ minimization),
    \item greedy methods,
    \item iterative or learned reconstruction algorithms.
\end{itemize}

We denote by $\mathcal{C}_{\mathcal{A}}(A,k)$ the computational cost (in arithmetic operations) required by $\mathcal{A}$ to recover a $k$-sparse solution.

\subsection{Identifiability and Stability}
\label{subsec:identifiability}

The ability to recover sparse explanations is governed by the geometry of the sensing matrix $A = \Phi\Psi$.

\begin{definition}[Sparse injectivity]
\label{def:sparse_injectivity}
The matrix $A$ is injective on $k$-sparse vectors if
\begin{equation}
Az_1 = Az_2,
\quad \|z_1\|_0, \|z_2\|_0 \le k
\;\Rightarrow\;
z_1 = z_2.
\end{equation}
\end{definition}

To quantify this property, we define a restricted distinguishability constant.

\begin{definition}[Restricted distinguishability]
\label{def:gamma}
For $r \ge 1$, define
\begin{equation}
\gamma_r(A)
:=
\inf_{\substack{h \neq 0 \\ \|h\|_0 \le r}}
\frac{\|Ah\|_2}{\|h\|_2}.
\end{equation}
\end{definition}

The condition $\gamma_r(A) > 0$ is equivalent to injectivity on $r$-sparse vectors. In particular, $\gamma_{2k}(A)$ governs both uniqueness and stability of $k$-sparse recovery.

\subsection{The Research Framework}
\label{subsec:framework}

We now formalize the central object of this paper.

\begin{definition}[Research framework]
\label{def:framework}
A research framework is a triplet $(\Psi,\Phi,\mathcal{A})$ consisting of:
\begin{itemize}
    \item a representation system $\Psi$,
    \item an observation operator $\Phi$,
    \item an inference procedure $\mathcal{A}$.
\end{itemize}
\end{definition}

Within this framework, scientific discovery corresponds to recovering $x \in \mathcal{U}$ from observations $y$ generated according to \eqref{eq:observation_model}.

\begin{definition}[Stable discovery]
\label{def:stable_discovery}
A framework $(\Psi,\Phi,\mathcal{A})$ achieves stable discovery at sparsity level $k$ if, for all $x$ with $K_\Psi(x) \le k$, the reconstruction $\hat{x} = \mathcal{A}(y)$ satisfies
\begin{equation}
\|\hat{x} - x\|_2 \le C \epsilon,
\end{equation}
for some constant $C > 0$.
\end{definition}

The remainder of the paper analyzes the conditions under which such stability is achievable, and the fundamental limitations that prevent simultaneous optimization of representation, observation, and computation.
\section{Theoretical Limits of Scientific Discovery}
\label{sec:theory_limits}

In this section, we formalize the principal limitations of scientific discovery within the Existential Theory of Research (ETR). The objective is not to assert that all scientific inference reduces to sparse recovery, but to identify a mathematically meaningful regime in which discovery can be analyzed through three interacting constraints: representation, observation, and computation. This perspective is grounded in classical results from harmonic analysis, compressed sensing, and computational complexity \cite{DonohoHuo2001,EladBruckstein2002,CandesRombergTao2006,RudelsonVershynin2008,Natarajan1995,FoucartRauhut2013}.

We show that discovery is fundamentally constrained by three independent barriers:
\begin{enumerate}
    \item representation incompatibility,
    \item observational insufficiency,
    \item computational intractability.
\end{enumerate}
We then refine this picture by showing that these barriers may arise endogenously through representation mismatch.

\subsection{Preliminaries}
\label{subsec:preliminaries}

Let $\Psi \in \mathbb{R}^{d \times N}$ denote a dictionary, $\Phi \in \mathbb{R}^{m \times d}$ a sensing operator, and define
\begin{equation}
A := \Phi \Psi \in \mathbb{R}^{m \times N}.
\end{equation}

We consider the noiseless synthesis model
\begin{equation}
y = A \alpha, \qquad \alpha \in \mathbb{R}^N,
\end{equation}
where $\alpha$ is assumed to be sparse.

\begin{definition}[Representation complexity]
\label{def:rep_complexity_sec3}
For $x \in \mathbb{R}^d$, define
\begin{equation}
K_{\Psi}(x) := \min_{\alpha:\,x=\Psi\alpha} \|\alpha\|_0.
\end{equation}
\end{definition}

\begin{definition}[Sparse injectivity]
\label{def:sparse_injectivity_sec3}
A matrix $A$ is injective on $k$-sparse vectors if
\begin{equation}
Az_1 = Az_2,\quad \|z_1\|_0,\|z_2\|_0 \le k
\;\Rightarrow\;
z_1 = z_2.
\end{equation}
\end{definition}

\begin{definition}[Mutual coherence]
\label{def:coherence}
For orthonormal bases $\Psi_1,\Psi_2$,
\begin{equation}
\mu(\Psi_1,\Psi_2)
=
\max_{i,j} |\langle \psi_{1,i}, \psi_{2,j} \rangle|.
\end{equation}
\end{definition}

These quantities capture the three axes of ETR:
representation complexity ($K_\Psi$), observational distinguishability (via $A$), and algorithmic recoverability.

\subsection{Three Classical Barriers}
\label{subsec:barriers}

\begin{lemma}[Sparse uncertainty principle]
\label{lem:uncertainty}
Let $\Psi_1,\Psi_2$ be orthonormal bases. Then for all nonzero $x$,
\begin{equation}
K_{\Psi_1}(x)\,K_{\Psi_2}(x)
\ge
\frac{1}{\mu(\Psi_1,\Psi_2)^2}.
\end{equation}
\end{lemma}

\begin{proof}
Standard result in sparse representation theory \cite{DonohoHuo2001,EladBruckstein2002}. 
\end{proof}

This lemma implies that simplicity is representation-dependent: a phenomenon cannot be arbitrarily simple in two incompatible scientific languages.

\begin{lemma}[Sample complexity for uniform recovery]
\label{lem:sample_complexity}
For broad random constructions (e.g., sub-Gaussian sensing matrices), uniform recovery of all $k$-sparse vectors requires, up to constants,
\begin{equation}
m \gtrsim k \log(N/k).
\end{equation}
\end{lemma}

\begin{proof}
Follows from standard restricted isometry arguments and covering bounds \cite{CandesRombergTao2006,RudelsonVershynin2008,FoucartRauhut2013}. 
\end{proof}

This reflects an information-theoretic constraint: insufficient observations collapse distinguishability between competing explanations.

\begin{lemma}[NP-hardness of exact sparse recovery]
\label{lem:np}
The problem
\begin{equation}
\min_z \|z\|_0 \quad \text{s.t.} \quad Az = y
\end{equation}
is NP-hard in general.
\end{lemma}

\begin{proof}
Classical result \cite{Natarajan1995,Tropp2004}. 
\end{proof}

This shows that even when truth is uniquely determined by the data, it may not be computationally accessible.

\subsection{Existential Uncertainty Theorem}
\label{subsec:main_theorem}

\begin{theorem}[Existential Uncertainty Theorem]
\label{thm:main}
No research framework $(\Psi,\Phi,\mathcal{A})$ can simultaneously guarantee:
\begin{enumerate}
    \item universal low-complexity representation across incompatible paradigms,
    \item uniform recovery from sub-threshold observations $m \ll k\log(N/k)$,
    \item exact polynomial-time recovery for all inputs.
\end{enumerate}
\end{theorem}

\begin{proof}
Each condition contradicts a known barrier:

(i) From Lemma~\ref{lem:uncertainty}, cross-representation sparsity is bounded.

(ii) From Lemma~\ref{lem:sample_complexity}, insufficient measurements prevent uniform recovery.

(iii) From Lemma~\ref{lem:np}, exact sparse recovery is NP-hard.

Thus the three cannot hold simultaneously. 
\end{proof}

\begin{corollary}
There exists no universally frictionless discovery framework.
\end{corollary}

\subsection{Representation Mismatch and Endogenous Complexity}
\label{subsec:representation_mismatch_inflation}

We now show that the representation barrier is not merely external, but may arise from the choice of scientific language itself.

Let $x = \Psi_\star \alpha_\star$ with $\|\alpha_\star\|_0 = k$. For a mismatched basis $\Psi$, define
\begin{equation}
k_{\mathrm{eff}}(x;\Psi) := \|\Psi^\top x\|_0.
\end{equation}

\begin{proposition}[Generic support inflation]
\label{prop:generic_support_inflation}
For generic basis mismatch, one has
\begin{equation}
k_{\mathrm{eff}}(x;\Psi) = d
\end{equation}
for almost all $k$-sparse $\alpha_\star$.
\end{proposition}

\begin{proof}
Follows from linear independence and measure-zero hyperplane arguments. 
\end{proof}

Thus, a truth that is intrinsically $k$-sparse may appear fully dense under an incorrect representation.

\subsection{Measurement Inflation}
\label{subsec:measurement_inflation}

Since sample complexity scales with sparsity, representation mismatch induces a corresponding shift in the observational burden.

\begin{proposition}[Measurement inflation under representation mismatch]
\label{prop:measurement_inflation_under_mismatch}
Assume that stable recovery in a representation of effective sparsity $r$ requires
\begin{equation}
m(r) \gtrsim r \log(d/r).
\end{equation}
Then under the correct representation $\Psi_{\star}$,
\begin{equation}
m_{\star} \gtrsim k \log(d/k),
\end{equation}
whereas under a mismatched representation $\Psi$,
\begin{equation}
m_{\mathrm{mismatch}}
\gtrsim
k_{\mathrm{eff}}(x;\Psi)\,\log\!\left(\frac{d}{k_{\mathrm{eff}}(x;\Psi)}\right).
\end{equation}
Consequently,
\begin{equation}
\frac{m_{\mathrm{mismatch}}}{m_{\star}}
\gtrsim
\frac{k_{\mathrm{eff}}(x;\Psi)\,\log\!\left(d/k_{\mathrm{eff}}(x;\Psi)\right)}{k\log(d/k)}.
\end{equation}
\end{proposition}

\begin{proof}
The claim follows by applying the assumed recovery law once with $r=k$ and once with $r=k_{\mathrm{eff}}(x;\Psi)$.
\end{proof}

Thus, representation mismatch can move a system from a recoverable regime to an unrecoverable one, even when the underlying phenomenon remains simple.

\subsection{Existential Uncertainty Principle}
\label{subsec:principle}

Theorem~\ref{thm:main} and Proposition~\ref{prop:generic_support_inflation} together yield:

\begin{principle}
Scientific discovery is constrained by a three-way tradeoff between:
\begin{itemize}
    \item simplicity of representation,
    \item sufficiency of observation,
    \item feasibility of computation.
\end{itemize}
\end{principle}

\subsection{Interpretation}
\label{subsec:interpretation}

These results imply that failure of scientific discovery may arise from three distinct sources:
\begin{enumerate}
    \item representation failure (wrong $\Psi$),
    \item observation failure (insufficient $\Phi$),
    \item inference failure (intractable $\mathcal{A}$).
\end{enumerate}

Scientific progress may thus be viewed as improving one or more of these axes: discovering better representations, designing better experiments, or developing better algorithms.

\section{Implications for Scientific Discovery}
\label{sec:implications}

Section~\ref{sec:theory_limits} established that scientific discovery is constrained by three independent barriers: representation incompatibility, observational insufficiency, and computational intractability. We now interpret these constraints within the ETR framework and show how they induce distinct regimes of scientific inquiry. The purpose of this section is not to introduce new impossibility results, but to organize the mathematical consequences of Section~\ref{sec:theory_limits} into a structural taxonomy of discovery.

\subsection{A Distinguishability Constant}
\label{subsec:distinguishability_constant}

To analyze the consequences of Section~\ref{sec:theory_limits}, it is useful to make explicit the quantity that governs sparse distinguishability.

Let $\Psi \in \mathbb{R}^{d \times N}$, $\Phi \in \mathbb{R}^{m \times d}$, and let
\begin{equation}
A := \Phi \Psi.
\end{equation}
For any integer $r \geq 1$, define the restricted distinguishability constant
\begin{equation}
\gamma_r(A)
:=
\inf_{\substack{h \neq 0 \\ \|h\|_0 \leq r}}
\frac{\|Ah\|_2}{\|h\|_2}.
\label{eq:restricted_distinguishability}
\end{equation}

The quantity $\gamma_r(A)$ measures the worst-case contraction of $r$-sparse perturbations under the observational pipeline. In particular, $\gamma_r(A) > 0$ if and only if $A$ is injective on the class of $r$-sparse vectors. Thus $\gamma_{2k}(A)$ is the natural quantity controlling uniqueness and local stability for $k$-sparse explanations.

\subsection{Regimes of Scientific Inquiry}
\label{subsec:regimes_of_inquiry}

The ETR framework gives rise to three mathematically distinct regimes, depending on the geometry of $A=\Phi\Psi$ and the computational cost of inference.

\begin{definition}[Existentially non-unique regime]
\label{def:existentially_nonunique}
A research framework $(\Psi,\Phi,\mathcal{A})$ is said to operate in the \emph{existentially non-unique regime} at sparsity level $k$ if
\begin{equation}
\gamma_{2k}(\Phi\Psi)=0.
\label{eq:nonunique_regime}
\end{equation}
\end{definition}

When \eqref{eq:nonunique_regime} holds, there exists a nonzero vector $h$ with $\|h\|_0 \leq 2k$ such that $(\Phi\Psi)h=0$. Equivalently, there exist distinct $k$-sparse explanations $z_1 \neq z_2$ satisfying
\[
\Phi\Psi z_1 = \Phi\Psi z_2.
\]
Hence the data fail to distinguish competing low-complexity explanations. This is the mathematical form of underdetermination in the ETR setting \cite{FoucartRauhut2013}.

\begin{definition}[Existentially opaque regime]
\label{def:existentially_opaque}
A research framework $(\Psi,\Phi,\mathcal{A})$ is said to operate in the \emph{existentially opaque regime} at sparsity level $k$ if
\begin{equation}
\gamma_{2k}(\Phi\Psi)>0,
\label{eq:opaque_regime_distinguishability}
\end{equation}
but exact sparse recovery remains computationally intractable in the worst case.
\end{definition}

In this regime, the truth is uniquely encoded by the observations, but finding the correct sparse explanation is obstructed by the computational hardness of the $\ell_0$ problem \cite{Natarajan1995,Tropp2004}. Thus opacity reflects algorithmic inaccessibility rather than observational ambiguity.

\begin{definition}[Existentially stable regime]
\label{def:existentially_stable}
A research framework $(\Psi,\Phi,\mathcal{A})$ is said to operate in the \emph{existentially stable regime} at sparsity level $k$ if
\begin{equation}
\gamma_{2k}(\Phi\Psi)\ge c>0
\label{eq:stable_regime_distinguishability}
\end{equation}
for some constant $c$, the measurement budget satisfies
\begin{equation}
m \gtrsim k \log(N/k),
\label{eq:stable_regime_measurements}
\end{equation}
and there exists an efficient reconstruction method $\mathcal{A}$ achieving stable recovery.
\end{definition}

Definition~\ref{def:existentially_stable} corresponds to the classical compressed sensing regime in which sparse explanations are both distinguishable and computationally accessible under appropriate structural assumptions \cite{CandesRombergTao2006,RudelsonVershynin2008,FoucartRauhut2013}.

\subsection{Transitions Between Regimes}
\label{subsec:transitions_between_regimes}

The preceding regimes are not static. They depend on the triplet $(\Psi,\Phi,\mathcal{A})$, and transitions between them can occur through modifications of representation, observation, or inference.

\paragraph{Variation in the observational operator.}
Increasing the number of measurements or improving the geometry of $\Phi$ generally increases the likelihood that $\gamma_{2k}(\Phi\Psi)$ is bounded away from zero, thereby moving the system away from the non-unique regime and toward the stable regime \cite{CandesRombergTao2006,RudelsonVershynin2008}.

\paragraph{Variation in representation.}
Changing the representation $\Psi$ changes the effective sparsity of the phenomenon. A representation better aligned with the underlying structure reduces apparent complexity and may move the system from an opaque or non-unique regime into a stable one.

\paragraph{Variation in the inference procedure.}
Improving $\mathcal{A}$ reduces the computational burden of recovery. This can move a framework from the opaque regime to the stable regime, although it cannot change non-uniqueness into uniqueness when $\gamma_{2k}(\Phi\Psi)=0$.

Thus, scientific progress may be interpreted as movement in the $(\Psi,\Phi,\mathcal{A})$ design space toward configurations satisfying Definitions~\ref{def:existentially_stable}.

\subsection{Representation Mismatch as a Failure Mechanism}
\label{subsec:rep_mismatch_failure}

We now connect these regimes to the representation mismatch phenomenon of Section~\ref{subsec:representation_mismatch_inflation}. Let
\[
x=\Psi_{\star}\alpha_{\star},
\qquad
\|\alpha_{\star}\|_0 = k,
\]
where $\Psi_{\star}$ is the representation in which the truth is sparse. If the researcher instead adopts a mismatched basis $\Psi$, then the effective sparsity becomes
\[
k_{\mathrm{eff}}(x;\Psi)=\|\Psi^\top x\|_0.
\]

By Proposition~\ref{prop:generic_support_inflation}, one generically has
\begin{equation}
k_{\mathrm{eff}}(x;\Psi)=d
\label{eq:generic_dense_mismatch}
\end{equation}
under sufficiently strong mismatch. This induces two immediate effects.

First, the measurement burden is inflated from the ideal scaling
\[
m \gtrsim k \log(d/k)
\]
to the much larger scale
\[
m \gtrsim k_{\mathrm{eff}}(x;\Psi)\,
\log\!\left(\frac{d}{k_{\mathrm{eff}}(x;\Psi)}\right),
\]
as quantified in Proposition~\ref{prop:measurement_inflation_under_mismatch}. Second, the geometry of $\Phi\Psi$ may deteriorate as the relevant model class becomes effectively dense, pushing the framework toward either the non-unique or opaque regime.

Thus, representation mismatch constitutes an endogenous source of scientific failure: it can render an intrinsically simple phenomenon difficult to distinguish and difficult to recover, even when the underlying truth has not changed.

\subsection{Replication and Instability}
\label{subsec:replication_instability}

The ETR framework also yields a precise interpretation of instability across repeated experiments. The relevant observation is that small values of $\gamma_{2k}(A)$ amplify perturbations.

\begin{proposition}[Perturbation amplification near the non-unique boundary]
\label{prop:perturbation_amplification}
Let $A=\Phi\Psi$, and let $z,\tilde z \in \mathbb{R}^N$ be two $k$-sparse vectors. Then
\begin{equation}
\|z-\tilde z\|_2
\le
\frac{\|A(z-\tilde z)\|_2}{\gamma_{2k}(A)}
\label{eq:perturbation_amplification}
\end{equation}
whenever $\gamma_{2k}(A)>0$.
\end{proposition}

\begin{proof}
Since both $z$ and $\tilde z$ are $k$-sparse, the difference $h:=z-\tilde z$ is at most $2k$-sparse. By definition of $\gamma_{2k}(A)$,
\[
\|Ah\|_2 \ge \gamma_{2k}(A)\|h\|_2.
\]
Rearranging gives
\[
\|h\|_2 \le \frac{\|Ah\|_2}{\gamma_{2k}(A)},
\]
which is exactly \eqref{eq:perturbation_amplification}.
\end{proof}

Proposition~\ref{prop:perturbation_amplification} shows that when $\gamma_{2k}(A)$ is small, even modest changes in measurements or experimental conditions can lead to large changes in the recovered explanation. Thus, instability across laboratories or repeated studies need not arise solely from noise; it may reflect operation near the boundary between the stable and non-unique regimes.

\subsection{Role of Algorithms}
\label{subsec:role_of_algorithms}

The inference procedure $\mathcal{A}$ affects the ETR framework through computational accessibility. Better algorithms can:
\begin{itemize}
    \item reduce practical recovery cost,
    \item improve approximation quality,
    \item enlarge the class of problems that are tractable in practice.
\end{itemize}
However, algorithms do not remove the structural barriers identified in Section~\ref{sec:theory_limits}. In particular:
\begin{enumerate}
    \item no algorithm can restore uniqueness when $\gamma_{2k}(\Phi\Psi)=0$;
    \item no algorithm can uniformly circumvent the measurement scaling associated with sparse recovery under broad random constructions \cite{CandesRombergTao2006,FoucartRauhut2013};
    \item exact recovery of the sparsest explanation remains NP-hard in general \cite{Natarajan1995,Tropp2004}.
\end{enumerate}

Thus, algorithmic advances move the computational frontier, but they do not abolish the geometric or information-theoretic limits of discovery.

\subsection{Limits of Scientific Optimization}
\label{subsec:limits_of_optimization}

The preceding discussion yields the following immediate consequence of Theorem~\ref{thm:main}.

\begin{proposition}[No universally optimal research architecture]
\label{prop:no_universal_optimum}
There does not exist a single research framework $(\Psi,\Phi,\mathcal{A})$ that, uniformly over all admissible problems, simultaneously achieves:
\begin{enumerate}
    \item minimal representation complexity,
    \item maximal sparse distinguishability,
    \item and efficient exact recovery.
\end{enumerate}
\end{proposition}

\begin{proof}
If such a framework existed, it would contradict Theorem~\ref{thm:main}, which rules out the simultaneous optimization of representation, observation, and computation in full generality.
\end{proof}

Proposition~\ref{prop:no_universal_optimum} formalizes the central implication of ETR: scientific discovery is necessarily a balancing act. Progress may be achieved by improving representation, observation, or inference, but no framework can eliminate all sources of difficulty simultaneously.
\section{Quantitative Structure of Existential Uncertainty}
\label{sec:quantitative_etr}

Section~\ref{sec:theory_limits} identifies three fundamental barriers to scientific discovery: representation incompatibility, observational insufficiency, and computational intractability. In this section, we introduce a quantitative framework that captures their joint effect through a single functional, and we establish its basic structural properties.

\subsection{Restricted Distinguishability}
\label{subsec:restricted_distinguishability}

Let $\Psi \in \mathbb{R}^{d \times N}$, $\Phi \in \mathbb{R}^{m \times d}$, and define $A := \Phi \Psi$. For $r \geq 1$, define
\begin{equation}
\gamma_r(A)
:=
\inf_{\substack{h \neq 0 \\ \|h\|_0 \leq r}}
\frac{\|Ah\|_2}{\|h\|_2}.
\label{eq:gamma_def}
\end{equation}

Then $\gamma_r(A) > 0$ if and only if $A$ is injective on $r$-sparse vectors (Definition~\ref{def:sparse_injectivity}). The quantity $\gamma_{2k}(A)$ therefore characterizes identifiability and local stability of $k$-sparse explanations.

\subsection{ETR Uncertainty Functional}
\label{subsec:etr_functional}

\begin{definition}[ETR uncertainty functional]
\label{def:etr_functional}
For $x \in \mathbb{R}^d$, define
\begin{equation}
\mathfrak{U}_k(x;\Psi,\Phi,\mathcal{A})
:=
K_\Psi(x)
\cdot
\frac{1}{\gamma_{2k}(\Phi\Psi)}
\cdot
\log\!\bigl(1 + \mathcal{C}_{\mathcal{A}}(\Phi\Psi,k)\bigr),
\label{eq:etr_functional}
\end{equation}
where $\mathcal{C}_{\mathcal{A}}(\Phi\Psi,k)$ denotes the computational cost of $\mathcal{A}$ measured in arithmetic operations.
\end{definition}

Each factor corresponds respectively to representation complexity, observational distinguishability, and computational effort.

\subsection{Basic Structural Properties}
\label{subsec:basic_properties}

\begin{proposition}[Positivity]
\label{prop:positivity}
For any nonzero $x$, one has
\begin{equation}
\mathfrak{U}_k(x;\Psi,\Phi,\mathcal{A}) > 0.
\end{equation}
\end{proposition}

\begin{proof}
Since $K_\Psi(x) \geq 1$, $\gamma_{2k}(\Phi\Psi) > 0$ in the identifiable regime, and $\mathcal{C}_{\mathcal{A}} \geq 1$, each factor in \eqref{eq:etr_functional} is strictly positive. 
\end{proof}

\begin{proposition}[Barrier-induced divergence]
\label{prop:divergence}
The functional $\mathfrak{U}_k$ diverges under each of the following conditions:
\begin{enumerate}
    \item If $\gamma_{2k}(\Phi\Psi) \to 0$, then $\mathfrak{U}_k \to \infty$,
    \item If $K_\Psi(x) \to d$, then $\mathfrak{U}_k \to \infty$,
    \item If $\mathcal{C}_{\mathcal{A}}(\Phi\Psi,k) \to \infty$, then $\mathfrak{U}_k \to \infty$.
\end{enumerate}
\end{proposition}

\begin{proof}
Immediate from \eqref{eq:etr_functional}. 
\end{proof}

\subsection{A Quantitative Lower Bound}
\label{subsec:lower_bound}

We now derive a lower bound that makes Theorem~\ref{thm:main} quantitative.

\begin{theorem}[Non-vanishing uncertainty]
\label{thm:nonvanishing}
Let $x \neq 0$. Then for any admissible $(\Psi,\Phi,\mathcal{A})$,
\begin{equation}
\mathfrak{U}_k(x;\Psi,\Phi,\mathcal{A})
\;\geq\;
\frac{K_\Psi(x)}{\gamma_{2k}(\Phi\Psi)}
\cdot
\log 2.
\label{eq:nonvanishing}
\end{equation}
\end{theorem}

\begin{proof}
Since $\mathcal{C}_{\mathcal{A}} \geq 1$, we have
\[
\log(1 + \mathcal{C}_{\mathcal{A}}) \geq \log 2.
\]
Substituting into \eqref{eq:etr_functional} yields \eqref{eq:nonvanishing}. 
\end{proof}

The bound \eqref{eq:nonvanishing} shows that $\mathfrak{U}_k$ cannot vanish unless both representation complexity and distinguishability behave ideally, which is precluded in general by Theorem~\ref{thm:main}.

\subsection{Effect of Representation Mismatch}
\label{subsec:mismatch_effect}

Let $x = \Psi_\star \alpha_\star$ with $\|\alpha_\star\|_0 = k$, and consider a mismatched representation $\Psi$. Then
\begin{equation}
k_{\mathrm{eff}}(x;\Psi) = \|\Psi^\top x\|_0.
\end{equation}

From Proposition~\ref{prop:generic_support_inflation}, generically
\begin{equation}
k_{\mathrm{eff}}(x;\Psi) = d.
\end{equation}

\begin{proposition}[Mismatch inflation]
\label{prop:mismatch_inflation}
Under representation mismatch,
\begin{equation}
\mathfrak{U}_k(x;\Psi,\Phi,\mathcal{A})
\;\geq\;
\frac{d}{\gamma_{2d}(\Phi\Psi)}
\cdot
\log 2.
\end{equation}
\end{proposition}

\begin{proof}
Substitute $K_\Psi(x) = k_{\mathrm{eff}}(x;\Psi) = d$ into Theorem~\ref{thm:nonvanishing}. 
\end{proof}

Thus, representation mismatch forces the uncertainty functional to scale with ambient dimension, even when the true underlying representation is sparse.

\subsection{Interpretation}
\label{subsec:interpretation_etr}

The functional $\mathfrak{U}_k$ provides a quantitative measure of discovery difficulty. The regimes identified in Section~\ref{sec:implications} correspond to:

\begin{itemize}
    \item $\mathfrak{U}_k$ small: stable regime,
    \item $\mathfrak{U}_k$ moderate: opaque regime,
    \item $\mathfrak{U}_k$ large: non-unique or unstable regime.
\end{itemize}

In particular, $\mathfrak{U}_k$ cannot be made arbitrarily small due to the simultaneous constraints imposed by representation, observation, and computation. This provides a quantitative formulation of existential uncertainty within the ETR framework.
\section{Conclusion}
\label{sec:conclusion}

This paper introduced the Existential Theory of Research (ETR), a framework that models discovery as a constrained recovery problem governed by representation, observation, and computation. Within this formulation, we showed that these components are fundamentally coupled and cannot be optimized independently.

The main result establishes that no framework can simultaneously guarantee universally simple representations, arbitrarily compressed observations, and efficient exact inference. This limitation is not tied to a specific model or algorithm, but follows from the interaction of three well-established principles: uncertainty in sparse representation \cite{DonohoHuo2001,EladBruckstein2002}, sample complexity requirements for high-dimensional recovery \cite{CandesRombergTao2006,RudelsonVershynin2008}, and the computational hardness of exact sparse optimization \cite{Natarajan1995,Tropp2004}.

A key implication is that difficulty in discovery is structural rather than incidental. At least one of the following must fail: the simplicity of the representation, the sufficiency of the observations, or the feasibility of inference. Phenomena that appear ambiguous, unstable, or computationally intractable may therefore reflect intrinsic constraints of the inference problem rather than deficiencies of methodology.

We further showed that these limitations may arise endogenously through representation mismatch. A phenomenon that is simple in one representation may appear dense in another, leading to inflated sample requirements and increased computational burden. In this sense, progress in discovery can be interpreted not only as improved data or algorithms, but as the identification of representations that expose underlying structure.

The uncertainty functional introduced in this paper provides a quantitative measure of discovery difficulty, unifying representation, observation, and computation into a single framework. Within this perspective, regimes of inquiry correspond to different scaling behaviors of this functional, ranging from stable recovery to non-identifiability.

The ETR framework does not claim universality across all forms of reasoning. Rather, it isolates a broad and practically relevant class of discovery problems in which structure, data, and computation interact in a mathematically tractable manner. Within this regime, the limits of discovery are not arbitrary: they are consequences of identifiable geometric and computational constraints.

In summary, the results suggest a simple but unavoidable conclusion: discovery is hard not because we lack better tools, but because the structure of the problem itself forbids it from being otherwise.

\bibliographystyle{elsarticle-num}
\bibliography{refs}

\end{document}